\documentclass{article}

 \PassOptionsToPackage{numbers, compress}{natbib}

\usepackage[utf8]{inputenc} 
\usepackage[T1]{fontenc}    
\usepackage{hyperref}       
\usepackage{url}            
\usepackage{booktabs}       
\usepackage{amsfonts}       
\usepackage{nicefrac}       
\usepackage{microtype}      
\usepackage{amsmath,amsthm,amssymb,color}
\usepackage[utf8]{inputenc}
\usepackage[dvipsnames]{xcolor}
\usepackage{soul}
\usepackage{enumitem}
\usepackage[parfill]{parskip}
\usepackage{graphicx}
\usepackage{hyperref}
\usepackage[hang,flushmargin]{footmisc}
\usepackage{scrextend}
\usepackage[]{algorithm2e}
\usepackage{enumitem}   

\usepackage{caption}
\usepackage{subcaption}

\makeatletter
\newcommand*\bigcdot{\mathpalette\bigcdot@{.5}}
\newcommand*\bigcdot@[2]{\mathbin{\vcenter{\hbox{\scalebox{#2}{$\m@th#1\bullet$}}}}}
\makeatother

\usepackage[nottoc,notlot,notlof]{tocbibind}

\hypersetup{
	colorlinks,
	citecolor=blue,
	filecolor=blue,
	linkcolor=blue,
	urlcolor=blue
}
\newtheorem{theorem}{Theorem}[section]
\newtheorem{claim}[theorem]{Claim}
\newtheorem{lem}[theorem]{Lemma}
\newtheorem*{defn}{Definition}

\newtheorem*{rem}{Remark}

\newtheorem{example}[theorem]{Example}

\newcommand{\eps}{\varepsilon}

\newcommand{\cX}{\mathcal{X}}

\DeclareMathOperator*{\argmax}{argmax}

\DeclareMathOperator*{\marg}{marg}
\DeclareMathOperator{\sign}{sign}

\title{On the Perceptron's Compression}

\author{
  Shay Moran\thanks{School of Mathematics, Institute for Advanced Study, Princeton NJ.
  \texttt{shay.moran1@gmail.com}.} 
  \and
Ido Nachum\thanks{Department of Mathematics,
Technion--IIT.
  \texttt{idon@tx.technion.ac.il}.
  Support by ISF grant 1162/15.} 
  \and  
 Itai Panasoff\thanks{Department of Mathematics,
Technion--IIT.
  \texttt{itai.panasoff@gmail.com}.} 
   \and  
 Amir Yehudayoff\thanks{Department of Mathematics,
Technion--IIT.\texttt{amir.yehudayoff@gmail.com}.  Support by ISF grant 1162/15.}
} 

\begin{document}

\global\long\def\kl{\mathsf{KL}}
\global\long\def\aa{\mathcal{\alpha}}
\global\long\def\cA{\mathcal{A}}
\global\long\def\bb{\mathcal{B}}
\global\long\def\cc{\mathcal{C}}
\global\long\def\DD{\mathcal{D}}
\global\long\def\ff{\mathcal{F}}

\global\long\def\hh{\mathcal{H}}
\global\long\def\ii{\mathcal{I}}
\global\long\def\jj{\mathcal{J}}
\global\long\def\kk{\mathcal{K}}
\global\long\def\lll{\mathcal{L}}
\global\long\def\mm{\mathcal{M}}
\global\long\def\nn{\mathcal{N}}
\global\long\def\oo{\mathcal{O}}
\global\long\def\pp{\mathcal{P}}
\global\long\def\qq{\mathcal{Q}}
\global\long\def\rr{\mathcal{R}}
\global\long\def\ss{\;}
\global\long\def\uu{\mathcal{U}}
\global\long\def\vv{\mathcal{V}}
\global\long\def\ww{\mathcal{W}}
\global\long\def\xx{\mathcal{X}}
\global\long\def\yy{\mathcal{Y}}
\global\long\def\zz{\mathcal{Z}}
 \global\long\def\k{\mathcal{\kappa}}
 \global\long\def\r{\mathcal{\rho}}
\global\long\def\m{\mathcal{\mu}}
\global\long\def\n{\kappa}

\global\long\def\Wa{W_{\mathcal{A}}}
\global\long\def\CC{\mathbb{C}}
\global\long\def\Dd{\mathbb{D}}
\global\long\def\EE{\mathbb{\mathbb{E}}}
\global\long\def\FF{\mathbb{F}}
\global\long\def\KK{\mathbb{K}}
\global\long\def\LL{\mathbb{L}}
\global\long\def\NN{\mathbb{N}}
\global\long\def\PP{\mathbb{P}}
\global\long\def\QQ{\mathbb{Q}}
\global\long\def\RR{\mathbb{R}}
\global\long\def\SS{\mathbb{S}}
\global\long\def\TT{\mathbb{T}}
\global\long\def\UU{\mathbb{U}}
\global\long\def\VV{\mathbb{V}}
\global\long\def\WW{\mathbb{W}}
\global\long\def\XX{\mathbb{X}}
\global\long\def\YY{\mathbb{Y}}
\global\long\def\ZZ{\mathbb{Z}}
\global\long\def\ne{\neq}
\global\long\def\ge{\leq}
\global\long\def\e{\varepsilon}
 \global\long\def\dd{\delta}
 \global\long\def\ne{\neq}
\global\long\def\sigm{\stackrel[i=1]{m}{\sum}}
\global\long\def\sigmn{\stackrel[i=m+1]{n}{\sum}}

\newcommand{\ReLu}{\mathsf{ReLU}}
\newcommand{\E}{\mathop \mathbb{E}}

\global\long\def\wt  { w^{(t)} }
\global\long\def\wto { w^{(t-1)} }

\global\long\def\xi{x_i }
\global\long\def\xio{x_{(i-1)} }

\global\long\def\con{\subset}
 \global\long\def\sec{\cap}
 \global\long\def\un{\cup}
 \global\long\def\do{_{1}}
\global\long\def\dt{_{2}}
\global\long\def\a{\rightarrow}
\global\long\def\x{\times}
\global\long\def\s{\curvearrowright}
\global\long\def\at{\mapsto}
\global\long\def\br{\left(\right)}
\global\long\def\c{\circ}
\global\long\def\d{\bigcdot}

\newcommand{\dist}{\mathsf{dist}}

\newcommand{\mynote}[1]{{#1}}

\definecolor{DarkPurple}{rgb}{0.7,0.2,0.4}
\newcommand{\rnote}[1]{\mynote{\color{DarkPurple} Raef: {#1}}}
\newcommand{\new}[1]{{\color{red} #1}}
\newcommand{\shay}[1]{\mynote{\color{red} Shay: {#1}}}

\date{}

\maketitle

\begin{abstract}
We study and provide exposition to several phenomena that are related to the perceptron's compression.
One theme 
concerns modifications of the perceptron algorithm that yield 
better guarantees on the margin of the hyperplane it outputs.
These modifications can be useful in 
training neural networks as well, and
we demonstrate them 
with some experimental data.
{In a second theme, we deduce conclusions from
the perceptron's compression in various contexts.}
\end{abstract}

\newpage

\section{Introduction}

\label{sec:int}

The perceptron is an abstraction of a biological neuron that was introduced in the 1950's by Rosenblatt~\cite{perceptron},
and has been extensively studied in many works (see e.g.\
the survey~\cite{Mohri}).
It receives as input a list of real numbers
(various electrical signals in the biological case) and if the weighted sum of its input is greater than some threshold it outputs $1$ and otherwise $-1$ (it fires or not in the biological case). 

Formally, a perceptron computes a function of the form
$\sign(w \bigcdot x - b)$ where $w \in \RR^d$ is the weight vector, $b \in \RR$
is the threshold, $\bigcdot$ is the standard inner product,
and $\sign :\RR \to \{\pm 1\}$ is $1$
on the {non-negative} numbers.
It is only capable of representing binary functions that are induced by partitions of $\RR ^ d$ by hyperplanes.

%

\medskip

\begin{defn}
A map $Y: {\cal X} \a \{\pm 1\}$ over a finite set ${\cal X} \con \RR ^d $ is (linearly)\footnote{We focus on the linear case, when the threshold is $0$.
{A standard lifting that adds a coordinate with $1$
to every vector allows to translate the general (affine) case to the linear case.
This lifting may significantly decrease the margin;
e.g., the map $Y$ on ${\cal X} = \{999,1001\} \subset \RR$ defined by
$Y(999)=1$ and $Y(1001) = -1$ has margin $1$ in the affine sense,
but the lift to $(999,1)$ and $(1001,1)$ in $\RR^2$ yields very small margin in the linear sense.
This solution may therefore cause an unnecessary increase in running time.}
This tax can be avoided, for example, if one has prior knowledge of  $R=\max_{x \in {\cal X}} \|x \|$. 
In this case, setting the last coordinate to be $R$
does not significantly decrease the margin.
In fact, it can be avoided without any prior knowledge 
using the ideas in Algorithm~\ref{alg:Rindep} below.}
separable if there exists {$w\in\RR^d$} 
such that $\sign(w \bigcdot x)=Y(x)$ for all $x\in {\cal X}$.
When the Euclidean norm of $w$ is $\|w\|=1$,
the number $\marg (w,Y) = {  \min_{x\in {\cal X}} {Y(x) w\bigcdot x }}$ is the margin 
of $w$ with respect to $Y$.
The number $\marg(Y) = \sup_{w \in \RR^d : \|w\|=1} \marg(w,Y)$
is the margin of $Y$. We call $Y$ an $\eps$-partition if its margin is at least $\eps$.
\end{defn}

Variants of the perceptron (neurons) are the basic building blocks of general neural networks.
Typically, the sign function is replaced by some other activation function
(e.g., sigmoid or rectified linear unit $\ReLu(z) = \max \{0, z\}$).
Therefore, studying the perceptron and its variants may 
help in understanding neural networks,
their design and their training process.

\subsubsection*{Overview}
In this paper, we provide some
insights into the perceptron's behavior,
survey some of the related work,
deduce some geometric applications, and discuss their usefulness in other learning contexts.
{Below} is a summary of our results and a discussion of related work,
{partitioned to five parts numbered (i) to (v).}
Each of the results we describe highlights
a different aspect of the perceptron's compression 
{(the perceptron's output is a sum of small subset of examples)}.
For more details, definitions and references, see the relevant sections.

\textbf{(i) Variants of the perceptron (Section~\ref{sec:PerVar}).} 
The well-known perceptron algorithm 
(see Algorithm~\ref{alg:Per} below) is guaranteed to find a separating hyperplane in the linearly separable case. However, there is no guarantee on the hyperplane's margin compared
to the optimal margin $\eps^*$. 
This problem was already addressed in several works,
as we now explain (see also references within).
The authors of~\cite{links} and~\cite{stab} defined a variant 
of the perceptron that yields a margin of the form {$\Omega(\eps^*/R^2)$; 
see Algorithm~\ref{alg:AggPer} below.} 
The authors of \cite{crammer} defined the
passive-aggressive perceptron
algorithms that allow e.g.~to deal with noise, but
provided no guarantee on the margin of the output.
The authors of \cite{max_marg} 
defined a variant of the perceptron
that yields provable margin
under the assumption
that a lower bound on the optimal margin is known.
The author of \cite{approx_marg} designed the ALMA algorithm
and showed that it provides almost optimal margin 
under the assumption that the samples lie on the unit sphere.
It is worth noting that normalizing the examples
to be on the unit sphere may significantly alter the margin,
and even change the optimal separating hyperplane.  
The author of \cite{un} defined the minimal overlap algorithm which guarantees optimal margin but is not online since it knows the samples in advance.
Finally, 
the authors of~\cite{grad}
analyzed gradient descent for a single neuron
and showed convergence to the optimal separating hyperplane
under certain assumptions
(appropriate activation and loss functions).

We provide two new ideas that improve the learning process.
One that adaptively changes the ``scale''
of the problem and by doing so 
{improves the guarantee on the margin of the output}
(Algorithm~\ref{alg:Rindep}),
and one that yields almost optimal margin
(Algorithms~\ref{alg:PerInf}).

\textbf{(ii) Applications for neural networks (Section~\ref{sec:AppNN}).}  
Our variants of the perceptron algorithm are simple 
to implement, and can therefore be easily applied 
in the training process of general neural networks. 
We validate their benefits 
by training a basic neural network 
on the MNIST dataset.

\textbf{(iii) Convex separation (Section~\ref{sec:convSep}).} 
We use the perceptron's compression to prove a sparse separation lemma for convex bodies. This perspective also suggests a different proof 
of Novikoff's theorem on the perceptron's convergence~\cite{Nov}.
In addition, we interpret this sparse separation lemma in the language of game theory
as yielding sparse strategies in {a related zero-sum} game.

\textbf{(iv) Generalization bounds (Section~\ref{sec:GenBound}).} 
{An important aspect of a learning algorithm is its generalization capabilities;
{namely, its error on {new examples that are independent of the training set}
(see the textbook~\cite{SS-text} for background and definitions).}
We follow the theme of~\cite{graepe05},
and {observe} that even though the (original) perceptron algorithm does not yield an optimal hyperplane, it still generalizes.

\textbf{(v) Robust concepts (Section~\ref{sec:RobCon}).}  
The robust concepts theme presented by Arriaga and Vempala~\cite{Arr}
suggests focusing on well-separated data.
We {notice} that the perceptron
fits well into this framework;
specifically, that its compression yields efficient dimension reductions.
Similar dimension reductions were used in several previous works
(e.g.~\cite{Arr,bendavid04, dim1,dim2,dim3,dim4,dim5}). 

{\textbf{Summary.}
In parts (i)-(ii) 
we provide a couple of new ideas {for improving the training process} and explain
their contribution in the context of previous work.
In part (iii) we use the perceptron's compression
as a tool for proving geometric theorems.
We are not aware of previous works that studied
this connection.
Parts (iv)-(v) are mostly about presenting
ideas from previous works in the context of the
perceptron's compression.
We think that parts (iv) and (v) help to 
understand the picture more fully.
}


%
%
%
%
%
%

\section{Variants of the Perceptron}

\label{sec:PerVar}

{Deciding how to train a model from a list of input examples is a central
consideration in any learning process.
In the case of the perceptron algorithm
the input examples are traversed while maintaining
a hypothesis $w^{(t)}$ in a way that reduces the error on the current example:}
 
\begin{algorithm}[H]
%
%
\textbf{initialize:} $w^{(0)}=\vec{0}$ and $t=0$

  \While {$\exists i$ with $y_i w^{(t)} \bigcdot x_i \leq 0 $}{
  $w^{(t+1)}=w^{(t)}+y_i x_i$ \\
  $t = t+1$
 }
 return $w^{(t)}$
 \caption{The perceptron algorithm}
 \label{alg:Per}
\end{algorithm}

{
{Clearly,} the perceptron algorithm terminates 
whenever its input sample is linearly separable,
in which case its output represents a separating hyperplane. 
Novikoff analyzed the number of steps $T$ 
required for the perceptron to stop as a 
function of the margin of the input sample~\cite{Nov}.}

{The standard analysis of the perceptron convergence properties
uses the optimal separating hyperplane $w^*$
(later in Section~\ref{sec:convSep} we present an alternative analysis that
does not use it)}:
 \[w^*=\argmax_{w \in \RR^d : \|w\|=1  } \marg(w,S) ,\]
where we think of $S$ as the map from $\{x_1,\ldots,x_m\}$
to $\{\pm 1\}$ defined by $x_i \mapsto y_i$.\footnote{We assume that 
$S$ {is consistent with} a function (does not contain {identical points with opposite labels}).}
Novikoff's analysis consists of the following two parts.
Let $\eps^* = \marg(w^*,S)$ and~$R=\max_i \|\xi \|$.

\noindent
{\em Part I: The projection grows linearly in time.}
In each iteration, the projection of $w^{(t)}$ on $w^*$ grows 
by at least $\eps^*$,
since $y_i  x_i  \bigcdot  w^* \geq \eps^*$.  
By induction, we get $w^{(t)} \bigcdot w^* \geq \eps^* t$ for all $t \geq 0$.

\noindent
{\em Part II: The norm grows sub-linearly in time.}
In each iteration, 
$$\|\wt \|^2 = \|\wto\|^2+ 2y_i\xi\bigcdot \wto +\|\xi\|^2
\leq \|\wto\|^2 + R^2$$ 
(the term $ 2y_i\xi\bigcdot \wto$ is negative by choice).
So by induction $\|\wt \| \leq R \sqrt{t}$ for all $t$.

Combining the two parts, 
$$1 \geq \frac{\wt \bigcdot w^* }{\|\wt\| \|w^*\|} \geq \frac{\eps^*}{R}\sqrt{t},$$ 
{which implies that} the number of iterations of the algorithm is at most $(R/\eps^*)^2$.

As discussed in Section~\ref{sec:int}, Algorithm~\ref{alg:Per} has several drawbacks.
Here we describe some simple ideas that allow to improve it.
Below we describe three algorithms,
each is followed by a theorem that summarizes
its main properties.

In the following, $\cX\subset \RR ^d$ is a finite set,
$Y$ is a linear partition,
$\eps^* = \marg(Y)$ is the optimal margin,
and $R=\max_{x \in {\cal X}} \|x \|$
is the maximal norm of a point.

%
%
%
%

In the first variant that already appeared in~\cite{links,stab},
the suggestion is to replace
the condition $y_i w^{(t)} \bigcdot x_i <0 $ by
$y_i w^{(t)} \bigcdot x_i < \beta$
for some {a priori} chosen $\beta >0$.
that may change over time. 
{As we will see,} different choices of $\beta$ yield different guarantees.


\begin{algorithm}[H]
%
%
%
\textbf{initialize:} $w^{(0)}=\vec{0}$ and $t=0$

  \While {$\exists i $ with $y_i w^{(t)} \bigcdot x_i < \beta $}{
  $w^{(t+1)}=w^{(t)}+y_i x_i$ \\
  $t = t+1$
 }
 return $w^{(t)}$
 \caption{The $\beta $-perceptron algorithm}
 \label{alg:AggPer}
\end{algorithm}

\medskip

\begin{theorem}[\cite{links,stab}]\label{R-perc}
 The $\beta$-perceptron algorithm performs at most $\frac{2\beta +R^2}{(\eps^*)^2}$ updates and achieves a margin of at least $\frac{\beta\epsilon ^ *}{2\beta+R^2} $.

\end{theorem}

\begin{proof}

We only replaced that $<0$ condition in the while loop
by a $<\beta$ condition, for some $\beta > 0$.
As before, by induction
$$\|\wt\|^2 = 
\|\wto\|^2+ 2y_i\xi\bigcdot \wto +\|\xi\|^2
\leq (2 \beta+R^2)t$$ 
and
$$1 \geq \frac{\wt \bigcdot w^* }{\|\wt\| \|w^*\|} \geq \frac{\eps^*}{\sqrt{2\beta +R^2}}\sqrt{t}$$
where $R = \max_{i} \|x_i\|$. 
The number of iterations is thus at most $\frac{2\beta +R^2}{(\eps^*)^2}$.
In addition, by choice, for all $i$,
$$y_i w^{(t)} \bigcdot x_i \geq \beta.$$
So, since
$$\|w^{(t)}\| \leq \sqrt{(2 \beta+R^2)t}
\leq \frac{2\beta +R^2}{\eps^*},$$
we get
$$\marg (\wt,S) \geq  \frac{\beta\eps^*}{2\beta +R^2}.$$  
\end{proof}


To remove the dependence on $R$
in the output's margin above,
{we propose to rescale $\beta$ according to the observed
examples.}

\begin{algorithm}[H]

%
%
\textbf{initialize:} $w^{(0)}=\vec{0}$ and $t=0$ and $\beta = 0$

  \While {$\exists i$ with $y_i w^{(t)} \bigcdot x_i \leq \beta$}{
  $w^{(t+1)}=w^{(t)}+y_i x_i$ \\
  
  $t = t+1$\\
  \If { $\beta <  \|\xi \|^2$ } {  $\beta = 4\|\xi \|^2 $ 
 \\ }
 }
 return $w^{(t)}$
 \caption{The $R$-independent perceptron algorithm}
 \label{alg:Rindep}
\end{algorithm}

\begin{theorem}\label{R-perc}
 The $R$-independent perceptron algorithm performs at most $\tfrac{10 R^2}{(\eps^*)^2}$ updates and achieves a margin of at least $\frac{\eps^*}{3}$.
 \end{theorem}
 
 \begin{proof}
 This version of the algorithm guarantees a margin of 
$\eps^* /3$ coupled with a running time comparable to the original algorithm
{\em without} knowing $R$. 
Indeed, 
to bound the running time, observe that before a change in $\beta$ occurs, there could be at most $\frac{2\beta +R^2}{\eps ^2}$ errors
(as before for the relevant $\beta$ and $R$). 
The amount of changes in $\beta$ is at most
$\lceil\log (R/r)\rceil$, where $r=\min_i \|\xi \|$.
The overall running time is at most
\begin{align*}
 \sum_{k=1}^{\lceil\log (R/r)\rceil}     \frac{2\bigcdot 4 \left| x_{i_k} \right|^2 +(2 \left| x_{i_k} \right| )^2}{(\eps^*)^2} & 
 \leq 2 \bigcdot   \sum_{k=1}^{\lceil\log (R/r)\rceil}  \dfrac{3\bigcdot 4^k r^2}{\eps ^2} \\
 & \leq 6\bigcdot 4/3\bigcdot 4^{\lceil\log (R/r)\rceil} \dfrac{r^2}{\eps ^2} = O((R/\eps^*)^2).
\end{align*}
 \end{proof}

Finally, if one would like
to improve upon the $\tfrac{\eps^*}{3}$ guarantee,
we suggest {to change $\beta$
with time.}
To run the algorithm,
we should first decide
how well do we want to approximate
the optimal margin.
To do so, we need to choose the parameter $\alpha \in (1,2)$;
the closer $\alpha$ is to $2$,
the better the approximation is
{(see Theorem~\ref{inf-perc})}.

\begin{algorithm}[H]
%
%
%
\textbf{initialize:} $w^{(0)}=\vec{0}$ and $t=0$ and $\beta = 0$ and $\alpha \in (1,2) $

  \While {$\exists i $ with $y_i w^{(t)} \bigcdot x_{i} \leq \beta$}{
  $w^{(t+1)}=w^{(t)}+y_i x_{i}$ \\
 
  $t = t+1$\\
  
   $\beta = 0.5 ((t+1)^\alpha - t^\alpha -1)$ \\
  
 }
 return $w^{(t)}$
 \caption{The $\infty$-perceptron algorithm}
 \label{alg:PerInf}
\end{algorithm}  

\medskip

\begin{theorem}\label{inf-perc}
If $R \leq 1$, 
the $\infty$-perceptron algorithm performs at most  
$(1/\eps^*)^{2/(2-\alpha)}$
updates and achieves a margin of at least $\frac{\alpha \epsilon ^ *}{2}$.
\end{theorem}

\begin{proof}
For simplicity, we assume here that $R=\max_i \|\xi \|=1$.{
The idea is as follows. 
The analysis of the classical perceptron 
relies on the fact that $\|w^{(t)}\|^2\leq t$ in each step.
On the other hand, in an ``extremely aggressive'' version of the perceptron
that always updates, one can only obtain a trivial bound
$\|w^{(t)}\|^2\leq t^2$ (as $w^{(t)}$ can be the sum of $t$ unit vectors in the same direction).
The update rule in the version below is tailored so that a bound of
$\|w^{(t)}\|^2\leq t^\alpha$ for $\alpha\in(1,2)$ is maintained.}

Here we use that
for $t \geq2 $, 
$$\| \wt \|^2  \leq \|\wto \|^2 + (t^\alpha - (t-1)^\alpha -1) + \|\xi \|^2.$$ 
By induction, for all $t \geq 0$,
$$\| \wt \|^2  \leq t^\alpha.$$
This time
$$1 \geq \frac{\wt \bigcdot w^* }{\|\wt\| \|w^*\|} \geq  \frac{\eps^* t}{t^{\alpha/2}}.$$
So, the running time is at most
$(1/\eps^*)^{2/(2-\alpha)}$.

The output's margin is at least
\begin{align}
\frac{0.5((t+1)^\alpha - t^\alpha -1)}{t^{\alpha/2}}.
\label{eqn:deriv}
\end{align}
This is decreasing function for $t>0$,
since its derivative is at most zero (see Appendix~\ref{app:calc}).

Since $(t+1)^\alpha - t^\alpha \geq \alpha t^{\alpha-1}$ for $t \geq 0$,
the output's margin is at least
$$
0.5 \alpha \frac{  (1/\eps^*)^{2(\alpha-1)/(2-\alpha)} -1}{(1/\eps^*)^{
\alpha/(2-\alpha)} }
= 0.5 \alpha \eps^* - (\eps^*)^{\alpha/(2-\alpha)} .$$
%
%

So we can get arbitrarily close to the true margin by setting $\alpha = 2(1-\delta)$
for some small $0 < \delta < 0.5$ of our choice.
This gives margin
$$(1-\delta)\eps^* -(\eps^*)^{(2-\delta)/\delta}
\geq \eps^* \big(1-\delta - (\eps^*)^{1/\delta} \big).$$
The running time, however, becomes  $(1/ \eps^*)^{1/\delta}$.

When $\eps^*$ is very close to $1$,
the lower bound on the margin above may not be meaningful.
We claim that the margin of the output is still
close to $\eps^*$ even in this case.
To see this, let $\tilde w$ be a hyperplane with margin
$\tilde \eps = (1 - \delta \ln(1/\delta)) \eps^*$.
We can carry the argument above with $\tilde w$ instead of $w^*$,
and get that the margin is at least
$$\tilde \eps \big(1-\delta - (\tilde \eps)^{1/\delta} \big)
> (1-2 \delta - \delta \ln(1/\delta)) \eps^*.$$
So we can choose $\delta$ small enough, without knowing
any information on $\eps^*$, and get an almost
optimal margin.

\end{proof}

\begin{rem}
The bound on the running time is sharp,
as the following example shows.
Let  ${\cal X}=\{\left((\sqrt{1-\eps ^2} , \eps ),1\right),{ \left(( \sqrt{1-\eps ^2} , -\eps ),-1\right)\ }$.
These two points are linearly separated with margin $\Omega(\eps)$.
The algorithm stops after $\Omega \big((1/\eps)^{2/(2-\alpha)} \big)$  
iterations (if $\eps$ is small enough and $\alpha$ close enough to~$2$).
\end{rem}

\medskip

\begin{rem}
Algorithms \ref{alg:Rindep} and \ref{alg:PerInf} can be naturally combined to a {single} algorithm that 
arrives arbitrarily close to the optimal margin {without assuming that $R \leq 1$.}
\end{rem}

\section{Application for Neural Networks}

\label{sec:AppNN}

Our results explain some choices that are made in practice, and can potentially help to improve them.
Observe that if one applies gradient descent
{on} a neuron of the form $\ReLu(w \bigcdot x)$
with loss function of the form $\ReLu(\beta - y_x w \bigcdot x)$ with $\beta=0$
then one gets the same update rule as in the perceptron algorithm.
Choosing $\beta =1$ corresponds to using the hinge loss
to drive the learning process.
The fact that $\beta=1$ yields provable bounds on the
output's margin of a single neuron {suggests a formal evidence that supports the benefits} of the hinge loss.

Moreover, {in practice, $\beta$ is treated as a hyper-parameter and tuning it
is a common challenge that {needs to be} addressed in order to maximize performance.}
We {proposed} a couple of new options for choosing and 
updating $\beta$ throughout the training process
{that may contribute towards a more systematic approach for setting $\beta$} (see Algorithms~\ref{alg:PerInf} and~\ref{alg:Rindep}).
{Theorems~\ref{R-perc} and~\ref{inf-perc} explain} the theoretical advantages of these options
in the case of a single neuron.

We also provide some experimental data.
Our experiments verify that our suggestions for choosing $\beta$ {can} indeed yield better results.
We used the MNIST database~\cite{LeCun} of handwritten digits as a test case {with no preprocessing}.
We used a simple and standard neural network with one hidden layer consisting
of {800/300 neurons and 10 output neurons
{(the choice of 800 and 300 is the same as 
in Simard et al.~\cite{Simard} and Lecun et al.~\cite{LeCun}).}}
We trained the network by back-propagation (gradient descent).
The loss function of each output neurons of the form
$\ReLu(w \bigcdot G(x) )$ where $G(x)$ is the output of the hidden layer  
is $\ReLu( - y_x w \bigcdot G(x)  + \beta)$
for different $\beta$'s.
This loss function is $0$ if $w$ provides a correct and confident 
(depending on $\beta$) classification of $x$
and is linear in $G(x)$ otherwise. 
This choice updates the network even 
when the network classifies correctly but {with less than $\beta$ confidence.}
It has the added value of yielding simple 
and efficient calculations compared to other choices (like cross entropy
or soft-max).\footnote{
{
An additional added value is that with this loss function
there is a dichotomy, either an error occurred or not.
This dichotomy can be helpful in making decisions throughout the learning process.
For example, instead of choosing the batch-size to be of fixed size $B$,
we can choose the batch-size in a dynamic but simple way:
just wait until $B$ errors occurred.
}}

We tested four values of $\beta$ as shown in Figure 1.
In two tests, the value of $\beta$ is fixed in time\footnote{Time
is measured by the number of updates.} to be $0$ and $1$.
In two tests, $\beta$ changes with the time $t$
in a sub-linear fashion. This choice can be better understood
after reading the analysis of Algorithm~\ref{alg:PerInf}.
Roughly speaking, the analysis predicts that
$\beta$ should be of the form
$t^{1-c}$ for $c>0$, and that the smaller $c$ is,
the smaller the error will be.
This prediction is indeed verified in the experiments;
it is evident that choosing $\beta$ in a time-dependent manner
yields better results.
For comparison, the last row of the table shows 
the error of the two-layer MLP of the same size 
that is driven by the cross-entropy loss~\cite{Simard}. {In fact, our network of 300 neurons performed better than all the general purpose networks with 300 neurons even with preprocessing of the data
that appear in http://yann.lecun.com/exdb/mnist/.}

\begin{figure}
\centering

\begin{minipage}{.5\textwidth}
  \centering
\begin{tabular}{ |c|c| } 
 \hline
 & test error 
 \\
\hline \hline 
$\beta = 0$ & \text{no convergence} \\
 \hline
$\beta = 1$ & 1.5 \% \\
 \hline
$\beta \approx t^{0.4}$ & 1.44 \% \\
 \hline
$\beta \approx t^{0.75}$ & 1.35 \% \\
 \hline
 \hline
cross-entropy \cite{Simard} & 1.6 \% \\
 \hline
\end{tabular}
\vspace{0.1cm}
\caption{one hidden layer with 800 neurons}

\end{minipage}%
\begin{minipage}{.5\textwidth}
\vspace{0.4cm}
  \centering
 \begin{tabular}{ |c|c| } 
 \hline
 & test  error 
 \\
\hline \hline

$\beta \approx t^{0.75}$ & 1.49 \% \\
 \hline
 \hline
 mean square error \cite{LeCun} & 4.7 \% \\
 \hline
MSE, [distortions] \cite{LeCun} & 3.6 \% \\
 \hline
 deskewing  \cite{LeCun} & 1.6 \% \\
 \hline
\end{tabular}
\vspace{0.1cm}
\caption{one hidden layer with 300 neurons}
\end{minipage}

\end{figure}

{ Finally, a natural suggestion that emerges from 
our work is to add $\beta >0$ as
a parameter for each individual neuron in the network,
and not just to the loss function.
Namely, to translate the input to a $\ReLu$ neuron by $\beta$.
The value of $\beta$ may change during the learning process.
Figuratively, this can be thought of as 
``internal clocks''  of the neurons.}



\section{Convex Separation}

\label{sec:convSep}
Linear programming (LP) is a central paradigm in computer science
and mathematics.
LP duality is a key ingredient in many algorithms and proofs,
and is deeply related to von Neumann's minimax
theorem that is seminal in game theory~\cite{Neumann}.
Two related and fundamental geometric properties are
Farkas' lemma~\cite{Farkas},
and the following separation theorem.

\medskip

\begin{theorem}[Convex separation theorem]
For every non empty convex sets $K ,L\subset \RR^d$,
{precisely one} of the following holds:
(i) $\dist(K,L) = \inf \{ \|p-q\| : p \in K, q \in L \} = 0$, or (ii)
there is a hyperplane separating $K$ and $L$.
\end{theorem}

We observe that the following stronger version of
the separation theorem follows from the perceptron's compression
(a similar version of Farkas' lemma can be deduced as well).

%
%
%

\medskip

\begin{lem}[Sparse Separation]
\label{lem:Farkas}
For every non empty convex sets $K,L \subset \RR^d$ 
so that $\sup \{\|p - q\| : p \in K , q \in L\} =1$
and every $\eps >0$, one of the following holds:
\begin{enumerate}[label=(\roman*)]
\item $\dist(K,L) < \eps$.
\item 
There is a hyperplane $H = \{ x : w \bigcdot x = b\}$ 
separating $K$ from $L$ so that its normal vector is ``sparse'':

-- $\frac{w\bigcdot p - b}{\|w\|} > \frac{\eps}{30}$ for all $p \in K$,

-- $\frac{w\bigcdot q - b}{\|w\|} < - \frac{\eps}{30}$ for all $q \in L$, and

-- $w$ is a sum of at most $(10/\eps)^2$ 
points in $K$ and $-L$.
\end{enumerate}
\end{lem}


\begin{proof}

Let $K,L$ be convex sets and $\eps > 0$.
For $x \in \RR^d$, let $\tilde x$ in $\RR^{d+1}$ be the same as
$x$ in the first $d$ coordinates and $1$ in the last
(we have $\|\tilde x\| \leq \|x\|+1$).
We thus get two convex bodies $\tilde K$ and $\tilde L$
in $d+1$ dimensions (using the map $x \mapsto \tilde x$).

Run Algorithm~\ref{alg:AggPer} with $\beta=1$ on inputs that positively label $\tilde K$
and negatively label $\tilde L$.
This produces a sequence of vectors $w^{(0)},w^{(1)},\ldots$
so that $\|\wt\| \leq \sqrt{6t}$ for all $t$.
For every $t>0$, the vector
$w^{(t)}$ is of the form $w^{(t)} = k^{(t)} - \ell^{(t)}$
where $k^{(t)}$ is a sum of $t_1$ elements of $\tilde K$ 
and $\ell^{(t)}$ is a sum of $t_2$ elements of $\tilde L$
so that $t_1+t_2 = t$.
In particular, 
we can write 
$\frac{1}{t} w^{(t)} = \alpha^{(t)} p^{(t)} - (1-\alpha^{(t)}) q^{(t)}$
for $\alpha^{(t)} \in [0,1]$ where
$p^{(t)} \in \tilde K$ and $q^{(t)}  \in \tilde L$ 
({note that the last coordinate of $w^{(t)}$ equals $2\alpha^{(t)}-\frac{1}{2}$}).


{If the algorithm does not terminate after $T$ steps
for $T$ satisfying $\sqrt{6 / T} < \eps /4$ then it follows that
$\| \frac{1}{T} w^{(T)}\| < \eps/4$. In particular, $|\alpha^{(T)} - 1/2| < \eps/8$
and so
\begin{align*}
\frac{\eps}{4} 
& >  \|\alpha^{(t)} p^{(t)} - (1-\alpha^{(t)}) q^{(t)}\|  >  \frac{\| p^{(t)} - q^{(t)}\|}{2} - \frac{\eps}{4},
\end{align*}
which implies that $\dist(K,L) < \eps$.}

{In the complementing case, the algorithm stops after $T < (10/\eps)^2$ rounds.
Let} $w$ be the first $d$ coordinates of $w^{(T)}$
and $b$ be its last coordinate.
For all $p \in K$,
$$\frac{w \bigcdot p + b}{\|w\|} \geq
 \frac{1}{\|w^{(T)}\|} 
\geq \frac{1}{\sqrt{6T}} > \frac{\eps}{30} .$$  
Similarly, for all $q \in L$ we get $\frac{w \bigcdot q + b}{\|w\|} < - \frac{\eps}{30}$.
\end{proof}

The lemma is strictly stronger than the {preceding separation theorem}.
Below, we also explain how this perspective 
yields an alternative proof of Novikoff's theorem
on the convergence of the perceptron~\cite{Nov}.
It is interesting to note that the usual proof of the separation
theorem relies on a concrete construction of the separating
hyperplane that is geometrically similar to hard-SVMs.
The proof using the perceptron, however, does not include
any ``geometric construction'' and yields
a sparse and strong separator
(it also holds in infinite dimensional Hilbert space,
but it uses that the sets are bounded in norm).

\subsubsection*{Alternative Proof of the Perceptron's Convergence}

Assume without loss of generality
that all of examples are labelled positively (by replacing $x$ by $-x$ if necessary).
Also assume that $R = \max_i \|x_i\| = 1$.
As in the proof above, let $w^{(0)},w^{(1)},\ldots$
be the sequence of vectors generated by the perceptron
(Algorithm~\ref{alg:Per}).
Instead of arguing that the projection
on $w^*$ grows linearly with $t$, argue as follows.
The vectors $v^{(1)},v^{(2)},\ldots$ defined by $v^{(t)} = \frac{1}{t} w^{(t)}$
are in the convex hull of the examples and have norm at most
$\|v^{(t)}\|\leq \frac{1}{\sqrt{t}}$.
Specifically, for every $w$ of norm $1$ we have
$v^{(t)} \bigcdot w \leq \frac{1}{\sqrt{t}}$
and so there is an example $x$ so that $x \bigcdot w \leq \frac{1}{\sqrt{t}}$.
This implies that the running time $T$ satisfies $\frac{1}{\sqrt{T}} \geq \eps^*$
since for every example $x$ we have $x \bigcdot w^* \geq \eps^*$.

\subsubsection*{A Game Theoretic Perspective}

The perspective of game theory
turned out to be useful in several
works in learning theory (e.g.~\cite{freund,moran}).
The ideas above have a game theoretic interpretation as well.
In the associated game there are two players. 
A Point player whose
pure strategies are points $v$ in some finite
set $V \subset \RR^d$ so that $\max \{\|v\|: v \in V\} = 1$,
and a Hyperplane player whose pure strategies
are $w$ for $w \in \RR^d$ with $\|w\|=1$.
For a given choice of $v$ and $w$,
the Hyperplane player's payoff is of $P(v,w) = v \bigcdot w$ coins
(if this number is negative, then the Hyperplane player
pays the Point player). The goal of the Point player is thus to minimize
the amount of coins she pays.
A mixed strategy of the Point player is a distribution $\mu$ on $V$,
and of the Hyperplane player is a (finitely supported) distribution $\kappa$ on $\{w : \|w\|=1\}$.
The expected gain is 
$$P(\mu,\kappa) = \E_{(v,w) \sim \mu \times \kappa} P(v,w).$$

\begin{claim}[Sparse Strategies]
\label{clm:GameT}
Let $\eps^*$ be the minimax value of the game:
$$\eps^* = \sup_{\kappa} \inf_{\mu} P(\mu,\kappa) \geq 0.$$
{There is $T  \geq \frac{1}{3(\eps^*)^2}$ (if $\eps^*=0$ then $T=\infty$)
and }a sequence of mixed strategies
$\mu_1,\mu_2,\ldots,\mu_T$ of the Point player so that for all $t\leq T$,
the support size of $\mu_t$ is at most $t$ and
for every
mixed strategy $\kappa$ of the Hyperplane player,
$$P(\mu_t,\kappa) \leq \sqrt{3/t}.$$
\end{claim}

\begin{proof}

Let $v^{(t)} = \frac{1}{t} w^{(t)}/t$ be as in the proof of Lemma~\ref{lem:Farkas} above,
when we replace $K$ by $V$ and $L$ by $\emptyset$.
We can interpret $v^{(t)}$ as a mixed-strategy $\mu_t$ of the Point player
(the uniform distribution over some multi-subset of $V$
of size $t$).
Specifically, for every $\kappa$ and~$t>0$,
$$P(\mu_t,\kappa) = \E_{w \sim \kappa} v^{(t)} \bigcdot w
\leq \|v^{(t)}\| \leq \sqrt{3/t}.$$
Denote by $T$ the stopping time.
If $T=\infty$ then indeed $P(\mu_t,\kappa)$ tends to zero as $t \to \infty$.
If $T<\infty$, we have $v \bigcdot v^{(T)}\geq \frac{1}{T}$ for all $v \in V$.
We can interpret $v^{(T)}$ as a non trivial strategy for the Hyperplane player:
let $$\tilde w = \frac{v^{(T)}}{\|v^{(T)}\|}.$$
Thus, for every $\mu$,
$$P(\mu,\tilde w)\geq \frac{1}{T \|v^{(T)}\|} \geq
\frac{1}{\sqrt{3T}}.$$
In particular, $\eps^* \geq \frac{1}{\sqrt{3T}}$ and so
$$T \geq \frac{1}{3(\eps^*)^2}.$$
\end{proof}

{The last strategy in the sequence
$\mu_1,\mu_2,\ldots$
guarantees the Point player a loss of at most $3 \eps^*$}.
This sequence is naturally and efficiently generated by the perceptron algorithm
and produces a strategy for the Point player that is optimal
up to a constant factor. 
The ideas presented in Section~\ref{sec:PerVar} allow
to reduce the constant $3$ to as close to $1$ as we want,
by paying in running time (see Algorithm~\ref{alg:PerInf}).

\section{Generalization Bounds}

\label{sec:GenBound}

{Generalization is one of the key concepts
in learning theory.}
{One typically formalizes it by assuming that the input sample consists of  i.i.d.\ examples drawn} 
from an unknown distribution
$D$ on $\RR^d$ that are labelled by some unknown function 
$c: \RR^d \to \{\pm 1\}$.
{The algorithm is said to generalize if it outputs} an hypothesis 
$h: \RR^d \to \{\pm 1\}$ so that $\Pr_D[h \neq c]$ is as small as possible.

We focus on the case that $c$ is linearly separable.
A natural choice for $h$ in this case is given by hard-SVM;
namely, {the halfspace with maximum margin on the input sample.}
It is known that if $D$ is supported on points that
are $\gamma$-far from some hyperplane then
the hard-SVM choice generalizes well
(see Theorem~15.4 in \cite{SS-text}). The proof of this property of hard-SVMs uses Rademacher complexity.
 
We suggest that using the perceptron algorithm,
instead of the hard-SVM solution, yields a more general statement with a simpler proof.
The reason is that the perceptron 
can be interpreted as a sample compression scheme.

\medskip

\begin{theorem}[similar to \cite{graepe05}]
\label{thm:General}
Let $D$ be a distribution on $\RR^d$.
Let $c:\RR^d \to \{\pm 1\}$.
Let $x_1,\ldots,x_m$ be i.i.d.\ samples from $D$.
Let $S = ((x_1,c(x_1)),\ldots,(x_m,c(x_m))$. 
If  \begin{align}
\label{eqn:PP}
{\Pr_S} \left[  \marg(S)< \eps \right] < \delta/2
\end{align} 
for some $\eps,\delta > 0$, then 
$$\Pr_S \left[ \PP_{D}[\pi(S) \neq c]\leq 
50\frac{\log\left( \eps^2 m\right)+\log(2/\delta)}{\eps^2 m}
\right] \geq 1-\delta$$
where $\pi$ is the perceptron algorithm.
\end{theorem}

The theorem can also be interpreted of as a local-to-global statement in the following sense.
Assume that we know nothing of $c$, but we get a list of $m$
samples that are linearly separable with significant margin
(this is a local condition that we can empirically verify).
Then we can deduce that $c$ is close to being linearly separable.
The perceptron's compression allows to deduce
more general local-to-global statements, like bounding
the global margin via the local/empirical margins
(this is related to~\cite{Bar}).

Condition~\eqref{eqn:PP} holds when
the expected value of one over the margin
is bounded from above (and may hold when $c$ is not linearly separable).
This assumption is weaker than the assumption in~\cite{SS-text}
on the behavior of hard-SVMs
(that the margin is always bounded from below).

For the proof of Theorem \ref{thm:General} we will need the following.


\medskip

\begin{defn}[Selection schemes] 
{A \textit{selection scheme} of size $d$ consists of a
compression map $\kappa$ and a reconstruction map $\rho$ 
such that for every input sample $S$:
\begin{itemize}
\item $\kappa$ maps $S$ to a sub-sample
of $S$ of size at most $d$.

\item $\rho$ maps $\kappa(S)$ to a hypothesis $\rho(\kappa(S)): {\cal X} \to \{\pm 1\}$;
this is the output of the learning algorithm induced by the selection scheme.
 
\end{itemize}}
\end{defn}

Following Littlestone and Warmuth,
David et al.\ showed that every selection scheme does not overfit its data~\cite{david}:
{Let $(\kappa,\rho)$ be a {selection scheme of} size $d$.
Let $S$ be a sample of $m$ independent examples from an arbitrary distribution $D$ 
that are labelled by some fixed concept $c$,
and let $K(S) = \rho\left(\kappa\left(S\right)\right)$ be the output of the selection scheme.
For a hypothesis $h$, let $L_D(h) = \Pr_D[h \neq c]$ denote the true error of $h$ and
$L_S(h) = \frac{1}{m} \sum_{i=1}^m 1_{h(x_i) = c(x_i)}$ denote the empirical error of $h$.}

\medskip

\begin{theorem}[\cite{david}]
\label{thm:david}
{
For every $\delta>0$,}
$$\Pr_{S}\left[\lvert L_D\left(K\left(S\right)\right) - L_S\left(K\left(S\right)\right) \rvert\geq
\sqrt{\eps \cdot L_S\left(K\left(S\right)\right)} + \eps \right]\leq \delta,$$
where
$$\eps = 50\frac{d\log\left(m/d\right)+\log(1/\delta)}{m}.$$
\end{theorem}

\begin{proof}[Proof of Theorem~\ref{thm:General}]
{Consider the following selection scheme of size $1/\eps^2$ 
that agrees with the perceptron on samples with margin at least $\eps$:
If the input sample $S$ has $\marg(S) \geq \eps$, apply the perceptron 
(which gives a compression of size $1/\eps^2$).
Else, 
compress it to the emptyset and reconstruct it to some dummy hypothesis.
The theorem now follows by applying Theorem~\ref{thm:david}
on this selection scheme and by the assumption that
that $\marg(S) \geq \eps$ for $1-\delta /2$ of the space
(note that $L_S(K(S))=0$ when $\marg(S)\geq\eps$).}
\end{proof}


\section{Robust Concepts}
\label{sec:RobCon}

Here we follow the theme of robust concepts presented by Arriaga and Vempala~\cite{Arr}.
Let ${\cal X} \subset \RR^d$ be of size $n$
so that $\max_{x \in {\cal X}} \|x \|=1$.
Think of ${\cal X}$ as representing a collection of high resolution images.
As in many learning scenarios, some assumptions on the learning problem
should be made in order to make it accessible.
A typical assumption is that the unknown function to be learnt
belongs to some specific class of functions.
Here we focus on the class of all $\eps$-separated partitions of ${\cal X}$;
these are functions $Y : {\cal X} \to \{\pm 1\}$ that are linearly separable
with margin at least $\eps$.
Such partitions are called robust concepts in~\cite{Arr}
and correspond to ``easy'' classification problems.

Arriaga and Vempala
demonstrated the difference between robust concepts
and non-robust concept with the following analogy;
it is much easier to distinguish between ``Elephant'' and ``Dog''
than between ``African Elephant'' and ``Indian Elephant.''
They proved that random projections
can help to perform efficient dimension reduction
for $\eps$-separated learning problems
(and more general examples).
{They} also described ``neuronal'' devices for performing it,
and discussed their advantages.
Similar dimension reductions were used in several other works
in learning e.g.~\cite{dim1,dim2,dim3,dim4,dim5}.

We observe that the perceptron's compression allows
to deduce a {\em simultaneous} dimension reduction.
Namely, the dimension reduction works simultaneously for the entire
class of robust concepts.
{This follows from results in Ben-David et al.~\cite{bendavid04},
who studied limitations of embedding learning problems in
linearly separated classes.}

We now explain this in more detail.
The first step in the proof is the following theorem.

\medskip

\begin{theorem}[\cite{bendavid04}]
\label{num} 
The number of $\eps$-separated partitions of 
${\cal X}$ is at most $(2(n+1))^{1/\eps ^2}$.
\end{theorem}

\vspace{-0.3cm}

\begin{proof}
Given an $\eps$-partition of the set ${\cal X}$, the perceptron algorithm finds a separating hyperplane after making at most $1/ \eps ^2$ updates. It follows that every $\eps$-partition can be represented by a multiset of ${\cal X}$ together with the corresponding signs. 
The total number of options is at most $(n+1)^{1/\eps ^2}  \bigcdot 2^{1/\eps ^2}$.
\end{proof}

\vspace{-0.3cm}

The theorem is sharp in the following sense.

\medskip

\begin{example}
Let $e_1,\ldots,e_n \in \RR^n$ be the $n$ standard unit vectors.
Every subset of the form $(e_i)_{i \in I}$ for $I \subset [n]$ of size
$k$ is $\Omega(1/\sqrt{k})$-separated, and there are ${n \choose k}$ such subsets.
\end{example}

The example also allows to lower bound the number of
updates of any perceptron-like algorithm.
If there is an algorithm that given $Y : {\cal X} \to \{\pm 1\}$
of margin $\eps$ is able to find $w$ so that
$Y(x) = \sign(w \bigcdot x)$ for $x \in {\cal X}$
that can be described by at most $K$ of the points in ${\cal X}$
then $K$ should be at least $\Omega(1/\eps^2)$.

The upper bound in the theorem allows to perform dimension reduction
that simultaneously works well on the entire concept class.
Let $A$ be a $k \times d$ matrix
with i.i.d.\ entries that are normally distributed ($N(0,1)$)\footnote{Other
distributions will work just as well.}
with $k \geq C \log (n/ \delta) / \eps^4$ where $C>0$
is an absolute constant.
Given $A$, we can consider
$$A {\cal X} = \{A x : x \in {\cal X}\} \subset \RR^k$$
in a potentially smaller dimension space.
The map $x \mapsto Ax$ is almost surely one-to-one on ${\cal X}$.
So, every subset of ${\cal X}$ corresponds to a subset of $A{\cal X}$
and vice versa.
The following theorem 
shows that it preserves all
well-separated partitions.

\medskip

\begin{theorem}[implicit in \cite{bendavid04}]
\label{thm:DimRed}
With probability of at least $1-\delta$ over the choice of $A$,
all $\eps$-partitions of ${\cal X}$ are $\eps /2$-partitions 
of $A{\cal X}$ and  
all $\eps/2$-partitions of $A{\cal X}$ are $\eps /4$-partitions 
of ${\cal X}$.
\end{theorem}

The proof of the above theorem 
is a simple application
of Theorem~\ref{num} together with the Johnson-Lindenstrauss lemma.

\medskip

\begin{lem}[\cite{JL}]
Let $x_1,...,x_N \in \RR ^d$ with $\|x_i\|\leq 1$ for all $i \in [N]$. Then, for every $\eps >0$
and $0 < \delta<1/2$,
$$\PP \Bigl[ \exists i,j \in [N] ~\left|  (A x_i\bigcdot Ax_j) - (x_i\bigcdot x_j) \right| > \eps\Bigr]<\delta,$$
where $k=O(\log (N/ \delta) / \eps ^2)$
and $A$ is a $k \times d$ matrix with i.i.d.\ entries that are $N(0,1)$.
\end{lem}
%

%

\appendix

\newpage

\medskip




\section{The derivative of the margin}

\label{app:calc}

Here we prove that the derivative of \eqref{eqn:deriv} is at most zero.
The numerator of the derivative is $0.5$ times
\begin{align*}
& (\alpha (t+1)^{\alpha-1}  - \alpha t^{\alpha-1})t^{\alpha/2}
- \frac{\alpha}{2} t^{(\alpha-2)/2}((t+1)^\alpha - t^\alpha-1)) \\
& = \frac{\alpha}{2} t^{(\alpha-2)/2} (2t (t+1)^{\alpha-1}  - 2 t^{\alpha})
+ \frac{\alpha}{2} t^{(\alpha-2)/2}(-(t+1)^\alpha + t^\alpha+1)) \\
& = \frac{\alpha t^{(\alpha-2)/2}}{2} \left( (t+1)^{\alpha-1} (t-1) - t^{\alpha}+1 \right) .
\end{align*}
At $t=1$, we get the value $0$,
so it suffices to prove that $(t+1)^{\alpha-1} (t-1) - t^{\alpha}+1$
is a non increasing function for $t\geq 1$.
Indeed, the derivative of the term inside the parenthesis is
\begin{align*}
& (\alpha-1)(t+1)^{\alpha-2} (t-1) + (t+1)^{\alpha-1}  - \alpha t^{\alpha-1} \\
& = (\alpha-1) \left(  \frac{t-1}{(t+1)^{2-\alpha}} 
-t^{\alpha-1} \right) + (t+1)^{\alpha-1}  - t^{\alpha-1} \\
& \leq (\alpha-1) \left(  \frac{t-1}{(t+1)^{2-\alpha}} 
-t^{\alpha-1} \right) + (\alpha-1) t^{\alpha-2} \tag{$\alpha  < 2$} \\
& \leq (\alpha-1) \left(  \frac{t-1}{t^{2-\alpha}} 
-t^{\alpha-1} + \frac{1}{t^{2-\alpha}} \right) =0 .
\end{align*}


\end{document}